\documentclass[accepted]{uai2026} % for initial submission
\usepackage[american]{babel}
\usepackage{natbib} % has a nice set of citation styles and commands
\usepackage{mathtools, amsmath, amssymb, amsthm, dsfont} % amsmath with fixes and additions
\usepackage{booktabs} % commands to create good-looking tables
\usepackage{tikz} % nice language for creating drawings and diagrams

\usepackage[ruled,vlined,linesnumbered]{algorithm2e}
\SetCommentSty{textnormal}                       

\let\oldnl\nl
\newcommand{\nonl}{\renewcommand{\nl}{\let\nl\oldnl}}

\usepackage{thmtools}
\usepackage{thm-restate}

\newtheorem{theorem}{Theorem}
\newtheorem{lemma}[theorem]{Lemma}

\newtheorem{remark}{Remark}
\newtheorem{assumption}{Assumption}

\DeclareMathOperator*{\argmin}{argmin}

\renewcommand{\S}{\mathcal{S}}
\newcommand{\A}{\mathcal{A}}

\newcommand{\norm}[1]{{\|#1\|}}
\newcommand{\proj}[1]{\mathrm{P}_{#1}}

\newcommand{\aggV}{\bar V}
\newcommand{\aggQ}{\bar Q}
\newcommand{\opt}{\mathrm{opt}}
\newcommand{\bellop}[1]{\mathcal{B}_{#1}}

\newcommand{\nodeset}{\mathcal{V}}
\newcommand{\edgeset}{\mathcal{E}}

\newcommand{\tree}{\mathcal{T}}
\newcommand{\leaf}{\mathrm{leaf}}
\newcommand{\leafNode}[1]{\nodeset_\leaf(#1)}
\newcommand{\expand}[1]{\mathrm{Expd}(#1)}
\newcommand{\collapse}[1]{\mathrm{Cllps}(#1)}
\newcommand{\est}{\mathrm{est}}
\newcommand{\children}[1]{\mathrm{Chldrn}(#1)}
\newcommand{\parent}[1]{\mathrm{Prnt}(#1)}

\title{Performance-Driven Environment Abstraction with Multi-Timescale Learning}

\author[1]{\href{mailto:<yguan44@gatech.edu>?Subject=Your UAI 2026 paper}{Yue Guan}}
\author[2]{\href{mailto:<dmaity@charlotte.edu>?Subject=Your UAI 2026 paper}{Dipankar Maity}}
\author[1]{\href{mailto:<tsiotras@gatech.edu>?Subject=Your UAI 2026 paper}{Panagiotis Tsiotras}}
\affil[1]{%
    School of Aerospace Engineering\\
    Georgia Institute of Technology\\
    Atlanta, GA, USA
}
\affil[2]{%
    Department of Electrical and Computer Engineering\\
    University of North Carolina at Charlotte\\
    Charlotte, NC, USA
}
  
\begin{document}

  \onecolumn
\maketitle

%%%%%%%%%%%%%%%%%%%%%%%%%%%%%%%%%%%%%%%%%%%%%%%%%%%%%%%%%
\begin{abstract}
    We study performance-driven environment abstraction for decision-making in large Markov decision processes.
    Rather than preserving geometric or topological structure, we seek abstractions that directly optimize decision quality.
    We model abstraction as a controlled approximation obtained by aggregating the state space and enforcing a shared action distribution within each aggregated state.
    For a fixed partition, we establish a performance guarantee that separates value-function approximation error from the loss introduced by action sharing.
    Guided by this analysis, we develop a multi-timescale reinforcement learning framework that jointly adapts the policy and a tree-structured environment abstraction.
    The resulting algorithm refines and coarsens regions of the state space based on Q-value discrepancies, balancing performance against abstraction size and complexity.
    Empirical results demonstrate substantial state compression, improved sample efficiency, and faster replanning compared to actor–critic baselines.
\end{abstract}

%%%%%%%%%%%%%%%%%%%%%%%%%%%%%%%%%%%%%%%%%%%%%%%%%%%%%%%%%
\section{Introduction}\label{sec:intro}
Humans and other intelligent species mitigate complexity by forming reduced representations of their environment, discarding irrelevant details while retaining task-relevant structure~\citep{eppe2022intelligent}.
A key mechanism is \emph{hierarchical representation}~\citep{hughes2024foundations}, where decisions are made at a coarse resolution and refined only when necessary.
By focusing reasoning at an appropriate level of abstraction, intelligent agents manage complex tasks while remaining adaptive to new tasks and evolving environments.
Consequently, the ability to automatically construct hierarchical representations online is widely regarded as a fundamental component of intelligence~\citep{simon1962architecture, botvinick2008hierarchical}.

This hierarchical reasoning capability is increasingly required in modern large-scale autonomy, where planning directly in the original state space becomes computationally infeasible and may fail to meet safety-critical reaction times.
Existing approaches often compress the environment using information-theoretic~\citep{larsson2020q,ravichandran2022hierarchical} or structural heuristics~\citep{machado2017eigenoption, dean1995decomposition, guan2022hierarchical}.
However, the appropriate abstraction is inherently task- and performance-dependent, as different tasks and performance requirements induce different compressed representations.

This paper studies \emph{performance-driven state abstraction}.
Rather than preserving geometry or topology, we seek abstractions that directly optimize decision quality.
We treat environment abstraction as a controlled state space aggregation, and we adaptively refine and coarsen the state space only when doing so improves agent's performance.
We derive performance guarantees identifying two sources of suboptimality:
(i) value-function approximation error and
(ii)~loss induced by enforcing a shared action distribution within each aggregated state.
Based on this analysis, we design a multi-timescale reinforcement learning scheme in which the policy converges under a slowly varying environment abstraction which evolves through a Q-value–guided refinement and aggregation mechanism.
The resulting algorithm continuously restructures the state space aggregation to maintain a compact representation while improving decision performance.
Empirically, the learned abstraction improves performance and transfers across related tasks.

\vspace{-0.1in}
\paragraph{Contributions.}
The contribution of this work is threefold:
(i) a performance bound for decision-making under an environment abstraction, which separates value approximation error from action-sharing error within aggregated states,
(ii)~a Q-value–based criterion providing principled refinement and aggregation rules for tree-based abstractions, and
(iii) a multi-timescale learning algorithm that jointly adapts the policy and the abstraction.

\subsection{Related Work}

\vspace{-0.1in}

A central line of work in hierarchical decision-making focuses on \emph{temporal abstraction}. 
The option framework~\citep{sutton1999between} and feudal hierarchies~\citep{dayan1992feudal} introduce multi-level policies in which high-level controllers select subgoals executed by lower-level policies until termination. 
Subsequent extensions~\citep{mcgovern2001automatic, csimcsek2005identifying, mahadevan2007proto} improve exploration and sample efficiency but operate over the original state space rather than compressing it.

In contrast, \emph{state space abstraction} methods aim to merge states into superstates and reason directly at the abstract level. 
Proto-value functions~\citep{mahadevan2007proto, machado2017eigenoption}, information-bottleneck-based tree searches~\citep{larsson2020q, larsson2023information} and heuristic-based methods~\citep{ravichandran2022hierarchical, hughes2024foundations} construct hierarchical representations that preserve topological or informational structure. 
However, these abstractions are typically domain-specific and not explicitly optimized for downstream task performance.

Closer to our work are \emph{performance-driven abstractions}, where partitions are adapted to improve value rather than preserve geometry. 
Baras and Borkar~\citep{baras2000learning} proposed a quantizer-based hierarchical actor–critic refined via multi-timescale stochastic approximation, later extended to continuous domains with value-based distances and entropy regularization~\citep{mavridis2021maximum}. 
While effective, these approaches do not explicitly analyze the performance loss induced by enforcing a shared action distribution within each aggregated state. 
We provide theoretical guarantees that separate value-function approximation error from action-sharing error, yielding a Q-value–guided aggregation criterion.

Tree-structured abstraction methods such as U-Tree~\citep{mccallum1996reinforcement, jonsson2000automated} and  conditional abstraction trees~\citep{dadvar2023conditional} adaptively split states using TD-error statistics. 
These approaches rely on iterative off-policy procedures, and support only refinement but not aggregation. 
In contrast, we jointly adapt the abstraction and policy via a unified multi-timescale scheme, justify Q-values as the structural criterion, and introduce a one-step lookahead/backward mechanism enabling principled refinement and collapse operations.

Overall, prior work either relies on fixed abstractions, adapts partitions heuristically without performance guarantees, or focuses on temporal rather than spatial abstraction.
Our framework addresses this gap by explicitly characterizing the performance loss and using it to guide principled refinement and aggregation within a unified learning scheme.

\vspace{-0.1in}
%%%%%%%%%%%%%%%%%%%%%%%%%%%%%%%%%%%%%%%%%%%%%%%%%%%%%%%%%
\section{Problem Formulation}
\vspace{-0.1in}
We consider a discounted Markov decision process (MDP)
$\langle \mathcal{S}, \mathcal{A}, R, P, \beta \rangle$,
where $\mathcal{S}$ and $\mathcal{A}$ are finite state and action spaces,
$P(s' \vert s,a)$ is the transition kernel,
$R(s)$ is the reward function, and $\beta \in [0,1)$ is the discount factor.
A stationary policy $\pi(a \vert s)$ induces the value function
\begin{equation}
    V^\pi(s)
    =
    \mathbb{E}_\pi\!\left[\sum_{t=0}^{\infty}\beta^t R(S_t)~\vert~ S_0=s\right].
\end{equation}
We aim to find a policy maximizing the performance $J(\pi;\mu_0)\!=\!\sum_{s}\mu_0(s)V^\pi(s)$,
under the initial state distribution $\mu_0$.

One way to construct such an optimal policy is through the optimal value function $V^*$, which is the unique fixed point of the Bellman operator
\begin{equation}
    \big(\bellop{\mathcal{S}}V^*\big)(s)
    \!=\!
    \max_{a\in\mathcal{A}}
    \Big[R(s)+\beta\sum_{s'\in\mathcal{S}} P(s' \vert s,a)V^*(s')\Big].
    \label{eq:bellman_full}
\end{equation}
However, evaluating~\eqref{eq:bellman_full} becomes computationally prohibitive when $|\mathcal{S}|$ is large, motivating reduced representations of the state space.

\paragraph{State aggregation.}
We approximate the decision problem using a \emph{state aggregation}
$\Gamma=\{\gamma_1,\dots,\gamma_m\}$,
where $\cup_{k=1}^m\gamma_k=\mathcal{S}$ and $\gamma_i\cap\gamma_j=\emptyset$ for $i\ne j$.
Each $\gamma_k\subseteq\mathcal{S}$ is called a \emph{superstate}, and the aggregation map $\phi_\Gamma:\mathcal{S}\to\Gamma$ assigns each state to its superstate.
Conceptually, $\Gamma$ defines an abstraction of the environment and serves as a reduced state representation for decision making.

Let $\pi_\Gamma: \Gamma \times\mathcal{A} \to [0,1]$ denote an aggregated policy.
Under this formulation, all states within the same superstate share the same action distribution,
which we refer to as the \emph{same-action-distribution (SAD)} constraint.
An aggregated policy always induces a policy on $\mathcal{S}$ via
\[
    \pi(a\vert s)=\pi_\Gamma(a\vert\phi_\Gamma(s)).
\]

For aggregated policies, decision-making is performed over $\Gamma$ instead of $\mathcal{S}$. 
This reduces the policy search space and allows executing policies  with ``minimum attention"~\citep{brockett1997minimum} where the state representation is ``coarse". 
Figure~\ref{fig:sad-illustration} illustrates the effect of the SAD constraint.

\begin{figure}[t]
    \centering
    \includegraphics[width=0.4\linewidth]{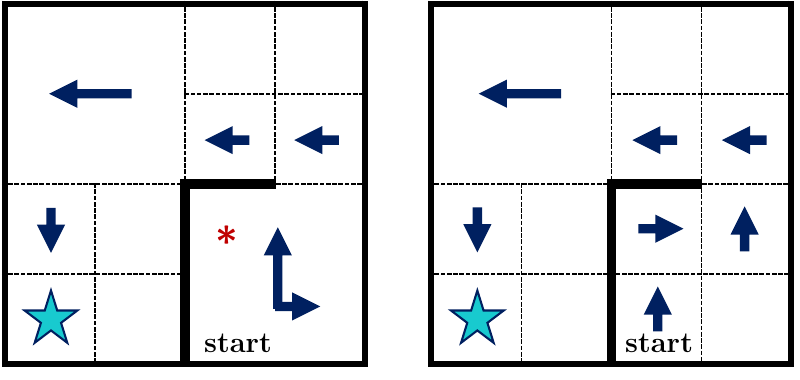}
    \vspace{-0.1in}
    \caption{
    Effect of the SAD constraint in a grid-world.
    \textbf{Left}: Coarse aggregation around the bottom-right forces action randomization, causing the agent to potentially stall at the red star.
    \textbf{Right}: Refining aggregation removes the need for randomization, enabling faster navigation.}
    \label{fig:sad-illustration}
\end{figure}

\paragraph{Abstraction--policy optimization.}
The above formulation leads to the following optimization problem, where we jointly select a state aggregation and aggregated policy, balancing performance and representation complexity:
\begin{equation}
    \max_{\pi_\Gamma,\,\Gamma}
    \; J(\pi_\Gamma;\mu_0) - \lambda |\Gamma|,
    \label{eqn:lagrangian-max}
\end{equation}
where $\lambda>0$ penalizes the number of superstates.

The above problem is inherently coupled.
The SAD constraint introduces approximation error whose magnitude depends on the chosen aggregation.
Conversely, evaluating the performance of an aggregated policy requires sufficient resolution of the value function in regions where actions differ.
This intrinsic coupling between $\Gamma$ and $\pi_\Gamma$ makes it challenging to solve~\eqref{eqn:lagrangian-max} directly.

%%%%%%%%%%%%%%%%%%%%%%%%%%%%%%%%%%%%%%%%%%%%%%%%%%%%%%%%%
\section{Learning Policies under Fixed Aggregations}
We introduce the aggregated Bellman operator and learn an optimal policy under a fixed aggregation $\Gamma$. 
We then quantify the performance loss induced by the SAD constraint.

Given a state distribution $\mu$, define the aggregated reward and transition kernel
\begin{align*}
    R_{\Gamma, \mu}(\gamma) &= \sum_{s\in\gamma} \mu_\gamma(s) R(s), \\
    P_{\Gamma, \mu}(\gamma' \vert \gamma, a)
    &=\sum_{s \in \gamma }\sum_{s' \in \gamma'}\mu_\gamma(s)P(s'\vert s,a),
\end{align*}
where $\mu_\gamma(s) = \mu(s)\big/\sum_{s\in\gamma}\mu(s)$ is the normalized distribution within $\gamma$.
Then, the performance of an aggregated policy $\pi_\Gamma$ evaluated over $\Gamma$ satisfies 
\begin{equation*}
    \aggV_{\Gamma, \mu
    }^{\pi_\Gamma}(\gamma) \!=\! R_{\Gamma, \mu}(\gamma) \!+\! \beta \!\sum_{a\in \A} \!\pi_\Gamma (a|\gamma)\!\sum_{\gamma' \in \Gamma} \!P_{\Gamma, \mu}(\gamma' |\gamma, a) \aggV_{\Gamma, \mu
    }^{\pi_\Gamma}(\gamma').
\end{equation*}

The optimal aggregated value function is the fixed point of
\begin{equation}
    \label{eqn:agg-fixed-point}
    (\bellop{\Gamma,\mu} \bar V)(\gamma)
    \!=\!
    \max_{a\in \A}
    \Big[
    R_{\Gamma,\mu}(\gamma)
    + \!\beta\! \sum_{\gamma' \in \Gamma}\!
    P_{\Gamma,\mu}(\gamma'\vert\gamma,a)
    \bar V(\gamma')
    \Big],
\end{equation}
and we refer to $\bellop{\Gamma,\mu}$ as the \emph{aggregated Bellman operator}, whose unique fixed point is denoted as $\aggV_{\Gamma, \mu}^*$.

\subsection{Aggregated Actor--Critic.}
To compute the optimal aggregated policy under a fixed partition $\Gamma$,
we employ a two-timescale actor--critic (AC) algorithm~\citep{borkar1997stochastic}.
The critic estimates the aggregated value function,
while the actor updates a softmax policy parameterization.

Let $\Gamma_t = \phi_\Gamma(S_t)$ denote the superstate containing $S_t$ at time~$t$.
The aggregated critic and actor updates are given by
\begin{align}
&\aggV_{t+1}(\Gamma_t)
=
(1-\xi_t)\aggV_t(\Gamma_t)
+ \xi_t \big(R_t + \beta \aggV_t(\Gamma_{t+1})\big), 
\label{eqn:agg-ac}
\\
&\aggQ_{t+1}\!(\Gamma_t,A_t\!)
\!=\!
\proj{q}\!\Big(\!
\aggQ_t(\Gamma_t,A_t\!)
\!+\! \zeta_t
\big(
\!R_t \!+\! \!\beta \aggV_t(\Gamma_{t+1}\!)
\!-\! \aggV_t(\Gamma_t)
\big)\!
\Big), \nonumber
\end{align}
where $\proj{q}$ projects onto $[-q,q]$ to ensure boundedness.

The action $A_t$ is sampled from the Boltzmann policy
\begin{equation}
\label{eqn:agg-bolzmann}
\psi_{\aggQ_t}(\Gamma_t,a)
=
\frac{e^{\aggQ_t(\Gamma_t,a)}}
{\sum_{a'} e^{\aggQ_t(\Gamma_t,a')}}.
\end{equation}

To ensure convergence, the critic operates on a faster timescale than the actor, i.e.,
$\lim_{t\to\infty}\zeta_t/\xi_t = 0$.
This on-policy scheme implicitly estimates
the state distribution $\mu$ used in~\eqref{eqn:agg-fixed-point} induced by $\pi_\Gamma$ through trajectory samples.

Under the following standard conditions on step sizes, convergence follows from classical two-timescale theory.

\begin{assumption}
\label{assump:stepsize}
    The learning rates $\{\xi_t\}$ and $\{\zeta_t\}$ satisfy
    \begin{align*}
    &\sum_{t=0}^{\infty} \xi_t = \sum_{t=0}^{\infty} \zeta_t = \infty, \quad
    \sum_{t=0}^{\infty} \xi_t^2 < \infty, \quad \sum_{t=0}^{\infty} \zeta_t^2 < \infty,\\
    &\lim_{t\to\infty} \xi_{t+1}/{\xi_t} = 1, \qquad
    \lim_{t\to\infty} {\zeta_t}/{\xi_t} = 0.
    \end{align*}
\end{assumption}

\begin{lemma}[\citep{baras2000learning}]
\label{lmm:agg-ac-convergence}
Under Assumption~\ref{assump:stepsize}, the aggregated actor--critic converges almost surely to $\aggQ^*$.
The converged actor induces a policy $\pi_\Gamma^*$ and stationary distribution $\mu^*$ satisfying
\[
\aggV_{\Gamma,\mu^*}^*(\gamma)
\le
\aggV_{\Gamma,\mu^*}^{\pi_\Gamma^*}(\gamma)
\le
\aggV_{\Gamma,\mu^*}^*(\gamma) + \epsilon(q),
\]
where $\epsilon(q)\to 0$ as $q\to\infty$.
\end{lemma}

\subsection{Performance Loss under fixed aggregations}
We now quantify the performance loss due to state aggregation by comparing
the optimal value function $V^*\in \mathbb{R}^{|\S|}$ with its aggregated counterpart
$\aggV^*_{\Gamma,\mu^*} \in \mathbb{R}^{|\Gamma|}$.
Because the two value functions lie on different domains, we introduce the value projection and its $\mu$-weighted pseudo-inverse
\begin{align*}
    &(\Psi_\Gamma \aggV{})(s) = \aggV{}(\phi(s)), \\
    &(\Psi^+_{\Gamma,\mu^*}V)(\gamma) = \sum_{s\in \gamma} \mu^*_\gamma(s) V(s).
\end{align*}

We then have the the following theorem, which quantifies the discrepancy between $\aggV_{\Gamma, \mu^*}^*$ and $V^*$, and provides the first insight toward an adaptive state aggregation algorithm.

\begin{restatable}{theorem}{thmaggopt}
    \label{thm:hier-sub-opt}
    Consider an MDP with state space $\S$ and a fixed aggregation $\Gamma$. 
    Let $V^* \in \mathbb{R}^{|\S|}$ denote the optimal value function of the MDP, and let $\aggV^*_{\Gamma, \mu^*} \in \mathbb{R}^{|\Gamma|}$ be the fixed point of~\eqref{eqn:agg-fixed-point}. 
    Then, we have
    \begin{equation}
        \label{eqn:agg-bound}
        \norm{\Psi_\Gamma \aggV^*_{\Gamma, \mu^*} - V^*}_{\infty} 
        \;\leq\; \frac{2 \epsilon_\Gamma + \delta_\Gamma}{1-\beta},
    \end{equation}
    where we have
    $\epsilon_\Gamma \!=\! \min_{\bar V \in \mathbb{R}^{|\Gamma|}} 
    \Vert \Psi_\Gamma \bar V - V^*\Vert_\infty$,
    and
    $\delta_\Gamma = \Vert \Psi^+_{\Gamma,\mu^*} \circ \bellop{\S} \circ \Psi_\Gamma \bar V_{\opt} 
    - \bellop{\Gamma,\mu^*} \bar V_{\opt}\Vert_\infty$
    with $\bar V_{\opt} \in \argmin_{\bar V \in \mathbb{R}^{|\Gamma|}} \norm{\Psi_\Gamma \bar V - V^*}_\infty$.
    % \dm{
    % I think we actually mean {\small 
    % \[ 
    % \delta_\Gamma = \max_{\bar V_{\opt} \in \argmin_{\bar V \in \mathbb{R}^{|\Gamma|}} \norm{\Psi_\Gamma \bar V - V^*}_\infty} \Vert \Psi^+_{\Gamma,\mu^*} \circ \bellop{\S} \circ \Psi_\Gamma \bar V_{\opt} 
    % - \bellop{\Gamma,\mu^*} \bar V_{\opt}\Vert_\infty
    % \]
    % }
    % } 
\end{restatable}

\begin{proof}
    The detailed proof is delayed to Appendix~\ref{appdx-sec:thm-sub-opt-proof}.
\end{proof}

\begin{remark}
\label{rmk:agg-epsilon}
The term $\epsilon_\Gamma$ measures the approximation error introduced by representing the optimal value function $V^*$ using piecewise-constant values over $\Gamma$.
% Equivalently, $V^*$ is approximated in the span of indicator basis vectors $e_\gamma \in \mathbb{R}^{|\S|}$, where $e_\gamma(s)=1$ if $s\in\gamma$ and $0$ otherwise.
Thus, $\epsilon_\Gamma$ captures the variation of $V^*$ within each superstate.
In particular, 
\[
\epsilon_\Gamma=\tfrac12 \max_{\gamma\in\Gamma}\Big(\max_{s\in\gamma}V^*(s)-\min_{s\in\gamma}V^*(s)\Big),
\]
and $\epsilon_\Gamma=0$ when $\Gamma=\S$.
\end{remark}

\begin{remark}
    \label{rmk:agg-delta}
    The $\delta_\Gamma$ term quantifies the performance loss due to the SAD constraint. 
    Specifically,
    \begin{align*}
        &\Big(\Psi^{+}_{\Gamma,\mu^*}\circ \bellop{\S} \circ \Psi_\Gamma \bar V_{\opt}\Big)(\gamma) = \sum_{s\in \gamma} \mu_\gamma(s) 
        \Big(
            \max_{a_s \in \A}\{
                R(s)\!+\! \beta\! \sum_{s' \in \S}\! P(s'\vert s,a_s) 
                \bar V_{{\opt}}(\phi(s'))
            \}
        \Big),
    \end{align*}
    which allows each $s \in \gamma$ to select its maximizing action independently with the best approximate value function $\aggV_{\opt}$.
    In contrast, the operator $\bellop{\Gamma,\mu^*}$ in~\eqref{eqn:agg-fixed-point} performs a single maximization at the superstate level, with all states in $\gamma$ sharing the same maximizing action.
    In the special case $\Gamma = \S$, the two operators coincide and $\delta_\Gamma = 0$.
\end{remark}

While Theorem~\ref{thm:hier-sub-opt} characterizes suboptimality, the quantities
$\epsilon_\Gamma$ and $\delta_\Gamma$ depend on the unknown optimal value $V^*$.
We therefore derive a computable upper bound based on Bellman residual mismatch.

\begin{restatable}{corollary}{coraggopt}
    \label{cor:hier-sub-opt}
    Consider an MDP and a fixed aggregation $\Gamma$, we have that
    \begin{equation}
        \label{eqn:cor-agg-bound}
        \Vert\Psi_\Gamma \aggV^*_{\Gamma, \mu} \!-\! V^*\Vert_{\infty} 
        \!\leq\! \frac{\Vert \Psi_\Gamma \bellop{\Gamma,\mu^*}\aggV^*_{\Gamma, \mu^*} \!-\! \bellop{\S } \Psi_\Gamma \aggV^*_{\Gamma, \mu^*}\Vert_\infty}{1-\beta}.
    \end{equation}
\end{restatable}

\begin{proof}
    The proof is presented in Appendix~\ref{appdx-sec:cor-sub-opt-proof}.
\end{proof}

In the sequel, we denote
\begin{equation*}
        \bar \epsilon_\Gamma = \norm{\Psi_\Gamma \bellop{\Gamma,\mu}\aggV^*_{\Gamma, \mu} - \bellop{\S } \Psi_\Gamma \aggV^*_{\Gamma, \mu}}_\infty.
\end{equation*}

As discussed in Remarks~\ref{rmk:agg-epsilon} and~\ref{rmk:agg-delta},
$\Gamma=\S$ yields optimal performance but no reduction in state space complexity,
whereas coarser partitions reduce complexity at the cost of suboptimality.
This trade-off motivates a multi-timescale algorithm that adaptively refines or aggregates regions of the state space to reduce estimated performance loss.

\section{Tree-based Adaptive State Aggregation}
Thus far, our analysis assumed a fixed aggregation $\Gamma$.
We now introduce a tree-structured scheme that adapts the aggregation online
while the actor--critic operates on a faster timescale.

We represent the state space $\S$ using a hierarchical tree
$\tree=(\nodeset,\edgeset)$.
Each leaf $\gamma\in\leafNode{\tree}$ defines a superstate,
thereby inducing a state aggregation.
We use $\tree$ interchangeably with the partition it induces.
Let $\tree_\S$ denote the finest tree whose leaves are singletons
(i.e., $\leafNode{\tree_\S}=\S$).
Throughout, we use quadtrees for concreteness,
although the method extends to general hierarchical trees.
An example of a quadtree is presented in Figure~\ref{fig:quad-tree}.

\paragraph{Tree operations.}
For a node $\gamma$, let $\parent{\gamma}$ denote its parent
and $\children{\gamma}$ its children.
A leaf $\gamma$ is \emph{expandable} if $\gamma\notin\leafNode{\tree_\S}$.
For $A\subseteq\leafNode{\tree}$,
the expansion operator $\expand{\tree,A}$ replaces each $\gamma\in A$
with its children.
A node $\gamma$ is \emph{collapsible}
if all its children are leaves of $\tree$.
For $B$ a set of collapsible nodes,
the collapse operator $\collapse{\tree,B}$
merges children back into their parent.
For example, in Figure~\ref{fig:quad-tree} one may obtain $\tree_\S=\expand{\tree,\{\gamma_2,\gamma_3\}}$.

\begin{figure}[t]
    \centering
    \includegraphics[width=0.6\linewidth]{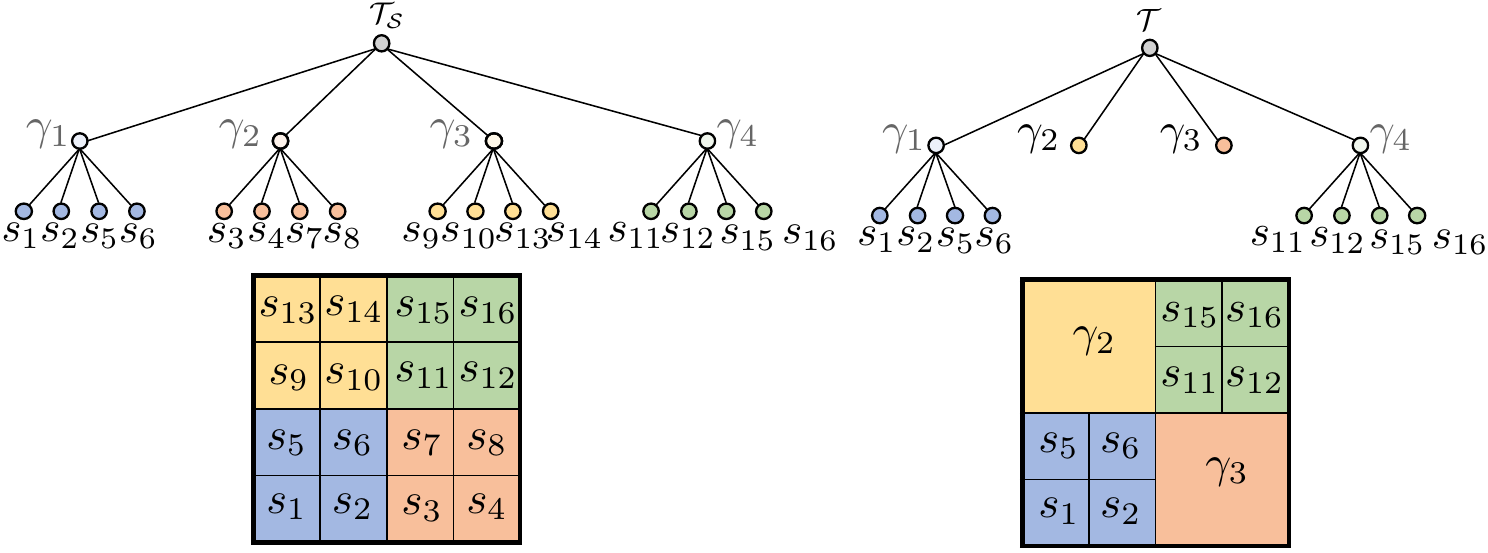}
    \caption{Tree representations of grid environments.}
    \label{fig:quad-tree}
    \vspace{-0.2 in}
\end{figure}

\subsection{Q-Estimation for Tree Updates}

The key research question in designing an adaptive state aggregation scheme is determining which nodes in tree $\tree$ should be expanded (refined) and which should be collapsed (coarsened). 
Our approach addresses this by maintaining Q-value estimates at the leaves and using them to guide dynamic refinement or coarsening of the partition.

Given a tree representation $\tree$ of the finite state space $\S$, we maintain a Q-value estimate at each leaf node, alongside the actor and critic, to evaluate policy performance locally. 
Suppose the process is currently in state $S_t$ with corresponding superstate $\Gamma_t = \phi_\tree(S_t) \in \leafNode{\tree}$. 
The Q-estimate at $\Gamma_t$ is updated as
\begin{align}
&\aggQ_{t+1}^\est(\Gamma_t,A_t)
=
(1-\eta_t)\aggQ_t^\est(\Gamma_t,A_t)
+ \eta_t \Big(
R_t 
+
\beta \big(
\mathds{1}_{S_{t+1}\in\Gamma_t}
\aggQ_t^\est(\Gamma_t,A_t)
 + 
\mathds{1}_{S_{t+1}\notin\Gamma_t}
\max_a \aggQ_t^\est(\Gamma_{t+1},a)
\big)
\Big),
\label{eqn:Q-est}
\end{align}
where $A_t$ is drawn according to~\eqref{eqn:agg-bolzmann} from the aggregated actor, and $\Gamma_{t+1}\in\leafNode{\tree}$ denotes the next superstate reached under $\tree$.
We update the estimated Q-values on a slower timescale than the actor to ensure that the partition remains approximately fixed (quasi-static) while the actor learns the aggregated policy under the current partition.
Specifically, we enforce that $\lim_{t\to\infty}\eta_t/\zeta_t=0$, where $\zeta_t$ is the step size for the actor update.

To assess whether a leaf node should be expanded, we maintain Q-estimates not only for the node itself but also for its prospective children.
The intuition is that the discrepancies between parent and child Q-values reflect the potential benefit of refining the representation in that region.
Formally, let $\Gamma_t \in \leafNode{\tree}$ be an expandable leaf, and define the locally expanded tree $\widetilde{\tree} = \expand{\tree,\Gamma_t}$.

Let $\widetilde{\Gamma}_t = \phi_{\widetilde{\tree}}(S_t)$ and $\widetilde{\Gamma}_{t+1} = \phi_{\widetilde{\tree}}(S_{t+1})$.
By construction, we always have $\widetilde{\Gamma}_t \in \children{\Gamma_t}$.
For the next time step, if $S_{t+1}\in\Gamma_t$, the agent remains within the expanded region, so $\widetilde{\Gamma}_{t+1}$ is either $\widetilde{\Gamma}_t$ as before or a sibling of $\widetilde{\Gamma}_t$.
If $S_{t+1}\notin \Gamma_t$, then the agent has left the expanded region, and $\widetilde{\Gamma}_{t+1}$ coincides with the superstate $\Gamma_{t+1}$ in the original process.
The Q-estimate for a child $\widetilde{\Gamma}_t \in \children{\Gamma_t}$ is then updated~as
\begin{align}
    \aggQ_{t+1}^\est (\widetilde{\Gamma}_t, A_t) 
    = (1-\eta_t)\,\aggQ_{t}^\est (\widetilde{\Gamma}_t, A_t) 
      + \eta_t \Big(R_t  +  \beta  \big(\mathds{1}_{S_{t+1}  \in \widetilde{\Gamma}_t} \aggQ_{t}^\est (\widetilde{\Gamma}_{t}, A_t)      
       &+  \mathds{1}_{S_{t+1} \in \Gamma_t  \setminus \widetilde{\Gamma}_t} \max_{a} \aggQ_{t}^\est (\widetilde{\Gamma}_{t+1}, a) 
       \label{eqn:Q-est-child-parent} \\
        & + \mathds{1}_{S_{t+1} \notin \Gamma_t}\max_a \aggQ_{t}^\est (\Gamma_{t+1}, a)\big)\Big), \nonumber
\end{align}
where the three indicators correspond to the three cases discussed above.  
An example of this update rule is shown in Figure~\ref{fig:q-est-learning-process}, which corresponds to the third case, where the agent leaves the parent $\Gamma_t$ after the transition. 

\begin{figure}[b]
    \centering
    \includegraphics[width=0.7\linewidth]{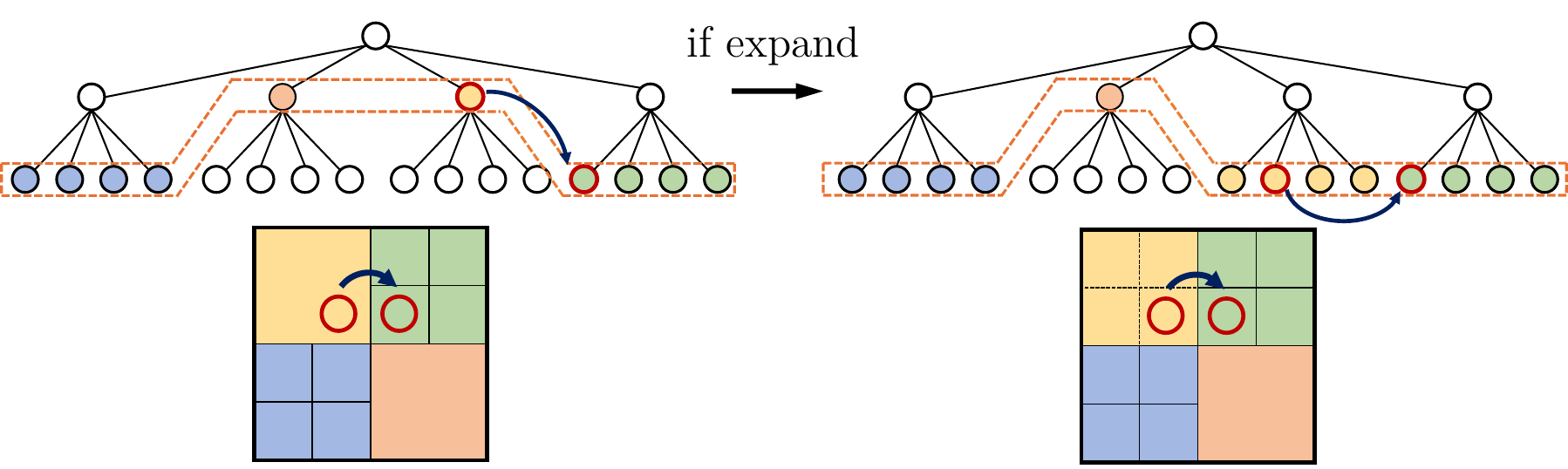}
    \caption{
    Illustration of the child Q-estimate update.
    Red circles mark the state transition and the corresponding superstates/nodes $\{\Gamma_t\}$ and $\{\widetilde{\Gamma}_t\}$ when the tree is expanded.
    }
    \label{fig:q-est-learning-process}
\end{figure}

For a collapsible parent $\Gamma'=\parent{\Gamma_t}$,
we define $\widetilde{\tree}=\collapse{\tree,\Gamma'}$
and update the merged node Q-estimate analogously to~\eqref{eqn:Q-est-child-parent}.

\begin{remark}
We use Q-estimates rather than critic values to guide structural updates
because the performance degradation in Theorem~\ref{thm:hier-sub-opt}
stems from the SAD constraint, which arises from the non-commutativity of maximization and averaging (see Remark~\ref{rmk:agg-delta}).
This effect is action-dependent and therefore cannot be detected from value estimates alone.
\end{remark}

\begin{remark}
Although the Q-estimate update in~\eqref{eqn:Q-est} appears independent of the aggregated actor–critic, it is indirectly coupled through the stationary distribution $\mu_\gamma$ within each superstate, which is induced by the actor–critic policy. 
Consequently, the Q-estimate cannot be learned independently in an off-policy manner.
\end{remark}

\subsection{Node Expansion and Collapse Criteria}

By Lemma~\ref{lmm:agg-ac-convergence}, the aggregated actor--critic
converges to the optimal aggregated policy for a fixed tree $\tree$.
Since the structural updates evolve on a slower timescale,
we adopt the standard multi-timescale argument:
when updating the partition, the actor--critic is treated as converged.
Thus, the structural optimization reduces to selecting a tree $\tree$,
with the corresponding optimal policy $\pi_\tree^*$ obtained implicitly.

We therefore consider the regularized objective
\begin{equation}
\label{eqn:lagrangian-max-q-est}
\max_{\tree}
\;\;
J(\pi_\tree^*;\mu_0)
-
\lambda |\leafNode{\tree}|.
\end{equation}

With only the aggregated value in hand, we use the approximation  $J(\pi^*_\tree;\mu_0) \approx \sum_{\gamma} \sum_{s \in \gamma} 
\mu_0(s)\aggV^*_{\tree,\mu^*}(\gamma)$.
Subtracting the constant $V^*(s)$ and bounding the difference by its absolute value yields the equivalent minimization
\[
\min_{\tree} 
\sum_{\gamma\in \leafNode{\tree}} \sum_{s\in\gamma} 
\mu_0(s)\big| \aggV^*_{\tree,\mu^*}(\gamma)-V^*(s)\big|
+ \lambda |\leafNode{\tree}|.
\]
Further bounding the weighted sum by the sup norm gives
\[
\min_{\tree} 
\norm{\Psi_\tree \aggV^*_{\tree,\mu^*} - V^*}_\infty
+ \lambda |\leafNode{\tree}|.
\]
Applying Corollary~\ref{cor:hier-sub-opt} leads to the final optimization problem:
\begin{equation}
    \label{eqn:tree-obj}
    \min_{\tree} ~
\frac{\bar \epsilon_\tree}{1-\beta}
+ \lambda |\leafNode{\tree}|,
\end{equation}
where $\bar \epsilon_\tree =
\norm{\Psi_\tree \bellop{\tree,\mu^*}\aggV^*_{\tree,\mu^*}
- \bellop{\S}\Psi_\tree \aggV^*_{\tree,\mu^*}}_\infty$.
This objective still captures the fundamental trade-off:
reducing abstraction error versus maintaining a compact tree.

\paragraph{Expansion criterion.}
Suppose the current quadtree is $\tree$, and let $\tree'$ be obtained by expanding a leaf $\gamma \in \leafNode{\tree}$. 
The error~\eqref{eqn:tree-obj} of the two trees are
\[
\frac{\bar \epsilon_\tree}{1-\beta} + \lambda |\leafNode{\tree}|
\quad\text{and}\quad
\frac{\bar \epsilon_{\tree'}}{1-\beta} + \lambda |\leafNode{\tree'}|.
\]
Since expanding one leaf in a quadtree adds three leaf nodes, 
a greedy one-step lookahead therefore expands $\gamma$ if
\[
\Delta \bar \epsilon_{\tree,\gamma}
:= \bar \epsilon_\tree - \bar \epsilon_{\tree'}
\;\ge\; 3(1-\beta)\lambda.
\]

Direct evaluation of $\Delta \bar \epsilon_{\tree,\gamma}$ requires the full model and state-level Bellman operator $\bellop{\S}$, which defeats the purpose of aggregation. 
Instead, we approximate the improvement using Q-estimates. 
Since $\Delta \bar \epsilon_{\tree,\gamma}$ captures the benefit of allowing children of $\gamma$ to adopt distinct action distributions (rather than a shared one), we approximate this gain by comparing the best action performance at $\gamma$ with the best performance attainable by its children:
\begin{equation}
    \label{eqn:expansion-error-epsilon}
    \hat \Delta \bar \epsilon_{\tree,\gamma}
    \!:= \!\!\!\!\!
    \max_{\gamma'\in\children{\gamma}} \max_{a\in\A} \aggQ_t^\est(\gamma',a)
    \!-\!
    \max_{a\in\A} \aggQ_t^\est(\gamma,a).
\end{equation}
This leads to the following expansion criterion:
\begin{equation}
\label{eqn:expansion-criteria}
\hat \Delta \bar \epsilon_{\tree,\gamma} \;\ge\; 3(1-\beta)\lambda.
\end{equation}

Upon expansion, the children superstates inherit actor and critic values from the parent via a soft update.
The \texttt{Expansion} routine is presented in Appendix~\ref{apdx-sec:code}.

\paragraph{Collapse criterion.}

Analogous to the expansion case, let $\tree'$ denote the tree obtained by collapsing a node $\gamma$.
Collapse is beneficial if
\(
\bar \epsilon_{\tree'} - \bar \epsilon_\tree
<
3(1-\beta)\lambda.
\)
Using the same Q-based approximation,
we collapse $\gamma$ when
\begin{equation}
\label{eqn:collapse-criteria}
\hat \Delta \bar \epsilon_{\tree',\gamma}
<
3(1-\beta)\lambda.
\end{equation}

When a node is collapsed, its actor-critic value is reinitialized from its children through a soft update, as specified in the \texttt{Collapse} routine in Appendix~\ref{apdx-sec:code}.

The complete learning procedure, integrating adaptive aggregation with aggregated actor–critic updates, is summarized in Algorithm~\ref{alg:abstraction-learning} in Appendix~\ref{apdx-sec:code}.

\section{Numerical Examples}

We evaluate the proposed method on both discrete navigation and continuous-control tasks and compare its performance against several learning-based baselines.

% \paragraph{Learned abstractions and policies.}
Figure~\ref{fig:abstraction-sequence} shows the sequence of abstractions learned until convergence at 15,000 episodes.
The color map represents the aggregated value function $\aggV$, with warmer colors indicating higher values.

The refinement process begins near the goal and propagates outward as value information spreads through the state space.
Furthermore, the algorithm adaptively refines the abstraction near obstacles and narrow corridors, allowing fine-grained maneuvers in critical regions while preserving coarser partitions elsewhere. 

\begin{figure}[h]
    \centering
    \includegraphics[width=0.7\linewidth]{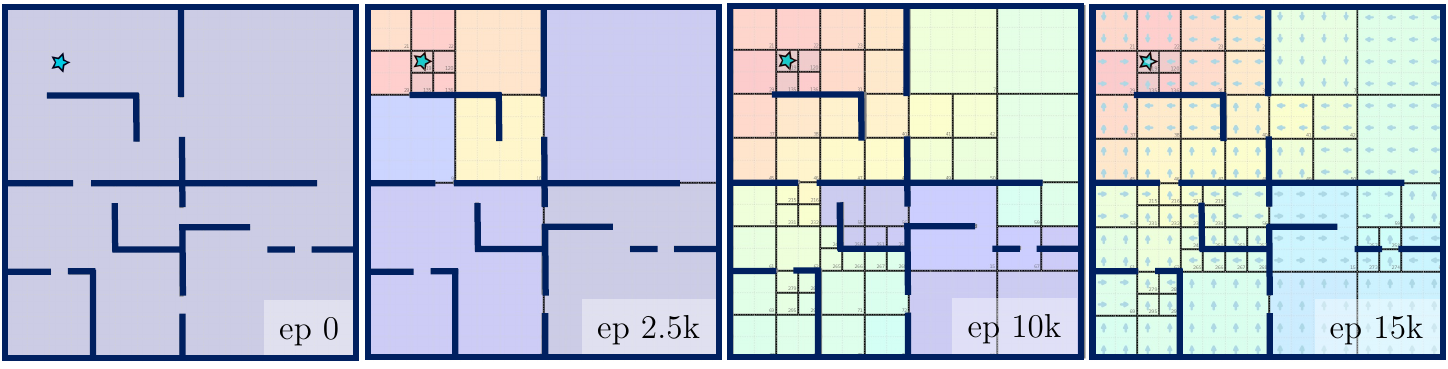}
    \vspace{-0.05in}
    \caption{
    Sequence of learned quadtrees in a $16\!\times\!16$ grid.
    % Colors indicate the aggregated value function.
    % Refinement concentrates near obstacles and narrow corridors.
    }
    \vspace{-0.1in}
    \label{fig:abstraction-sequence}
\end{figure}

% \paragraph{Reusing learned abstractions.}
We next relocate the goal from the top-left to the bottom-right corner as shown in Figure~\ref{fig:abstraction-sequence-replan}.
While the actor–critic values are reinitialized, the algorithm is warm-started with the previously learned abstraction. 
As a result, the algorithm converges much faster and quickly adapts by refining near the new goal and coarsening regions no longer requiring high resolution.
This illustrates the benefit of retaining task-relevant structure across related tasks.
Nonetheless, some regions near the previous goal remain overly refined due to exploration bias:
since action selections are biased toward fine-grained maneuvers, alternative actions are not adequately evaluated. 

\begin{figure}[h]
    \centering
    \includegraphics[width=0.63\linewidth]{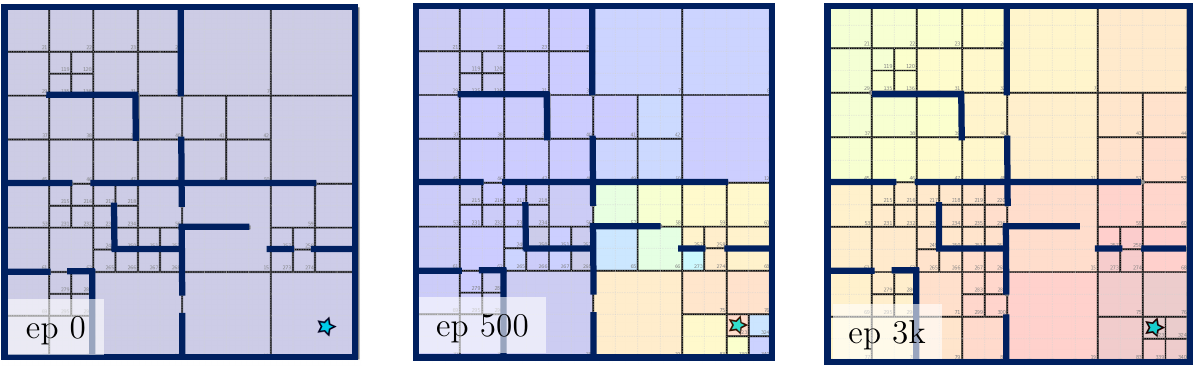}
    \vspace{-0.1in}
    \caption{Sequence of learned quadtree abstractions, starting with a previously learned abstraction. }
    % \vspace{-0.2in}
    \label{fig:abstraction-sequence-replan}
\end{figure}

Figure~\ref{fig:aggregation_examples} shows learned policies and state aggregations in randomly generated mazes.
Subplot (a) presents an $8\times8$ maze.
Subplots (b) and (c) depict a refined $16\times16$ version of the same maze augmented with two additional narrow corridors highlighted in cyan.
Abstractions and policies in (b) and (c) are learned with $\lambda=10$ and $\lambda=5$, respectively.
Subplot (d) shows a $32\times32$ maze.

\begin{figure}[t]
    % \vspace{-0.1in}
    \centering
    \includegraphics[width=0.6\linewidth]{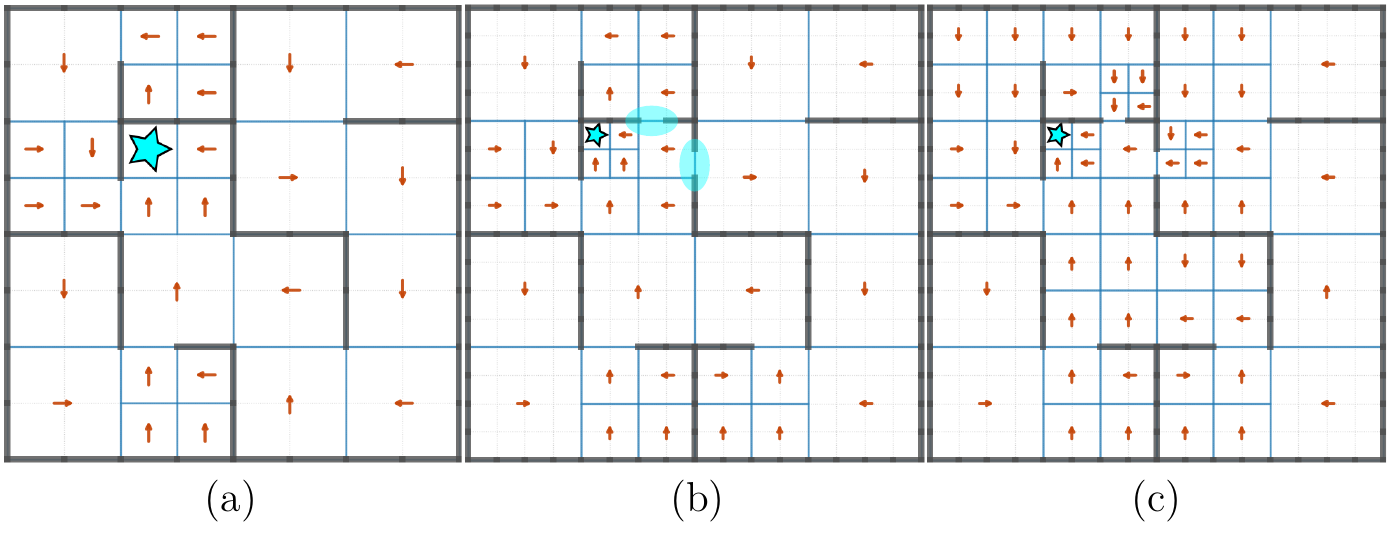}
    \vspace{-0.15in}
    \caption{
    Learned aggregated policies and state aggregations in random mazes.
    (a) $8\times8$ maze.
    (b)–(c) $16\times16$ maze with added corridors, learned with $\lambda=10$ and $\lambda=5$.
    }
    \vspace{-0.15in}
    \label{fig:aggregation_examples}
\end{figure}

The learned abstractions significantly compress the state space:
28 superstates out of 64 states in (a),
34/256 in (b),
55/256 in (c),
and 232/1024 in (d).
Despite this compression, the algorithm preserves task-relevant topological structures.
Relating to robotics applications, the $16 \times 16$ grid in (b) and (c) can be interpreted as a higher-resolution environment model of the $8 \times 8$ model in (a). 
% The policy and aggregation in subplot (b) are learned with penalty multiplier $\lambda = 10$, and (c) with $\lambda = 5$. 
The policies and partitions in (a) and (b) are nearly identical except near the goal, because the algorithm in (b), with $\lambda = 10$, ignores the narrow corridors and produces a suboptimal policy in exchange for fewer aggregated states. 
In contrast, the lower penalty ($\lambda = 5$) in (c) encourages use of the added corridors, yielding different policies that exploit the shorter paths. 

% In the larger $32\times32$ maze (d), refinement concentrates only in regions requiring precise maneuvers, demonstrating scalable and selective use of high-resolution representations.

Figure~\ref{fig:training-curves} compares the proposed method with a flat actor--critic baseline, the quantizer-based hierarchical actor–critic~\citep{baras2000learning}, and the CAT algorithm~\citep{dadvar2023conditional} across 40 randomly generated $16\times16$ navigation grids, a $128\times128$ Mars terrain map, and a discretized Mountain Car control problem.
We evaluate three variants of the proposed tree-based actor--critic algorithm:
\texttt{Tree-AC}, initialized with a single superstate corresponding to the root node;
\texttt{Tree-AC-Depth2}, initialized with a depth-2 tree; 
and
\texttt{Tree-AC-Replan}, warm-started from a previously learned partition after the goal location is changed.

Across all applicable scenarios, \texttt{Tree-AC-Replan} converges the fastest, demonstrating the benefit of reusing a previously learned abstraction to adapt to task changes.%
\footnote{
We do not evaluate \texttt{Tree-AC-Replan} on the Mountain Car task because the policy learned for the original goal can already reach all target positions.
}
Moreover, \texttt{Tree-AC-Depth2} converges faster than \texttt{Tree-AC}, demonstrating the benefit of the proposed approach’s flexibility in allowing the user to specify the initial abstraction granularity based on the task.

\begin{figure}[h]
    % \vspace{-0.1in}
    \centering
    \includegraphics[width=0.8\linewidth]{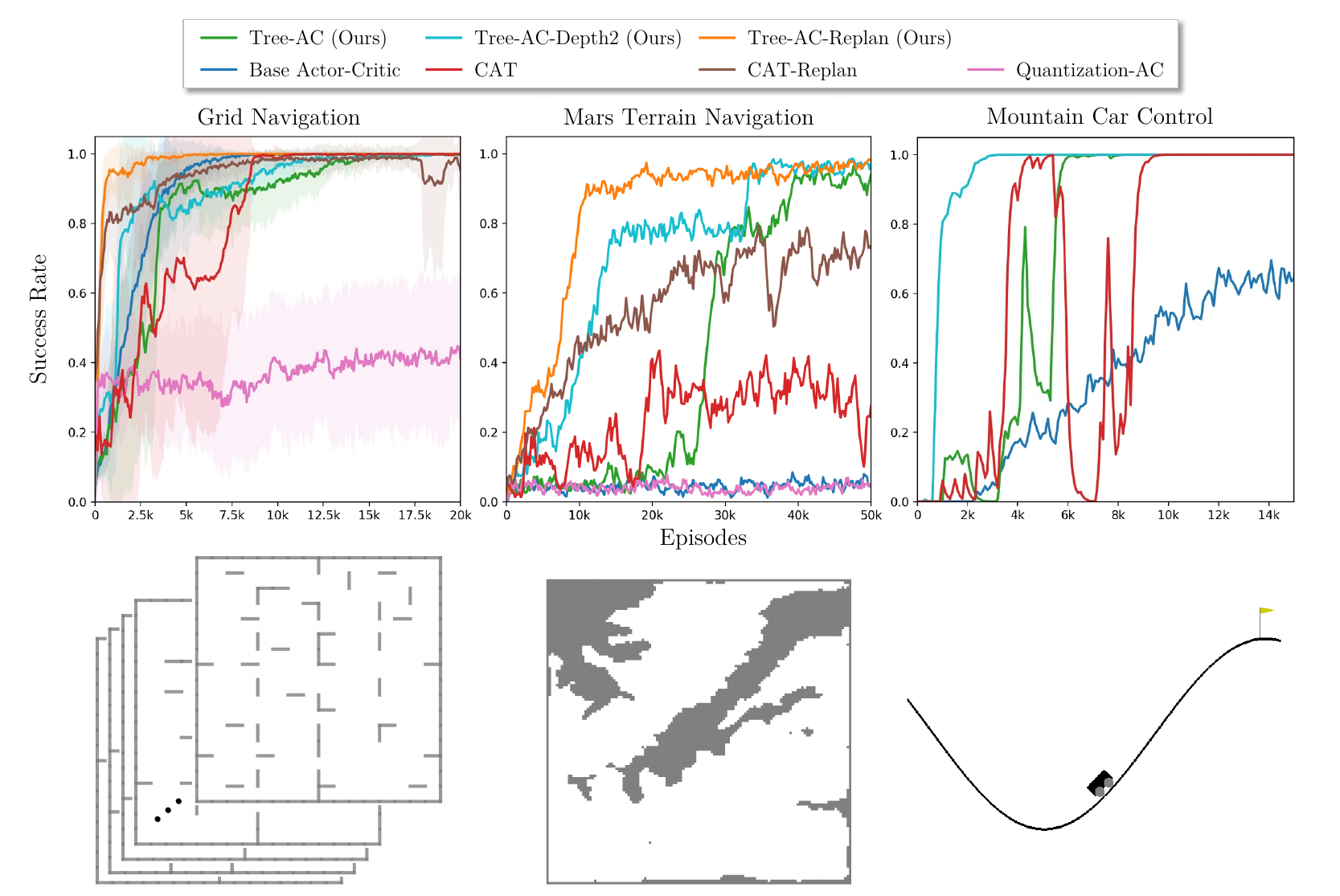}
    % \vspace{-0.1in}
    \caption{Training curves of algorithms on discrete navigation tasks and the discretized Mountain-Car control problem. }
    \label{fig:training-curves}
\end{figure}

In the $16\times16$ grid-navigation task, \texttt{BaseActorCritic} initially converges faster than \texttt{Tree-AC} initialized at depth 0, indicating that additional training is required to learn an effective abstraction.
The \texttt{Tree-AC-Depth2} variant learns faster than \texttt{BaseActorCritic} during the early stages of training but slows during the intermediate stage, as it must further refine the abstraction before converging to a high-performing policy.
The \texttt{Tree-AC-Replan} variant substantially outperforms \texttt{BaseActorCritic}, highlighting the benefit of reusing a learned environment representation when adapting to a new task.
Moreover, both \texttt{Tree-AC} variants outperform their corresponding \texttt{CAT} variants, as well as \texttt{Quantization-AC}.

The advantages of \texttt{Tree-AC} become more pronounced in the larger $128\times128$ Mars terrain-navigation task and the Mountain Car control problem.
In particular, \texttt{BaseActorCritic} struggles to learn effective policies in both settings, while \texttt{CAT} exhibits training instability, especially in the Mountain Car task.

Finally, we investigate how the initial tree depth affects the number of training episodes required to achieve a 95\% success rate in the grid-navigation problem.
The U-shaped trend in the left subplot of Figure~\ref{fig:pareto} reveals a trade-off between abstraction granularity and sample efficiency.
When initialized at a large depth, the algorithm begins with a fine-grained abstraction that enables more informative feedback but introduces a high-dimensional decision space, thereby slowing learning.
At the other extreme, aggregating all states into a single superstate provides insufficient resolution for collecting informative feedback and requires additional episodes to expand the abstraction to an effective level.
Consequently, initializing the algorithm at an intermediate tree depth minimizes the number of training episodes required to achieve a high success rate, consistent with the training curves shown in Figure~\ref{fig:training-curves}.

The right subplot of Figure~\ref{fig:pareto} presents the Pareto frontier between abstraction size and policy performance, highlighting the trade-off between the two.
As the abstraction becomes finer, the algorithm achieves a 99\% success rate before fully refining the policy.
Consequently, the two solutions with the largest abstraction sizes are dominated by smaller abstractions with comparable performance and therefore do not lie on the Pareto frontier.
At the other extreme, abstractions containing fewer than 20\% as many superstates as the original state space contains states fail to achieve a 60\% success rate.
These low-performing solutions are omitted from the plot for clarity.

\begin{figure}[b]
    \centering
    \includegraphics[width=0.9\linewidth]{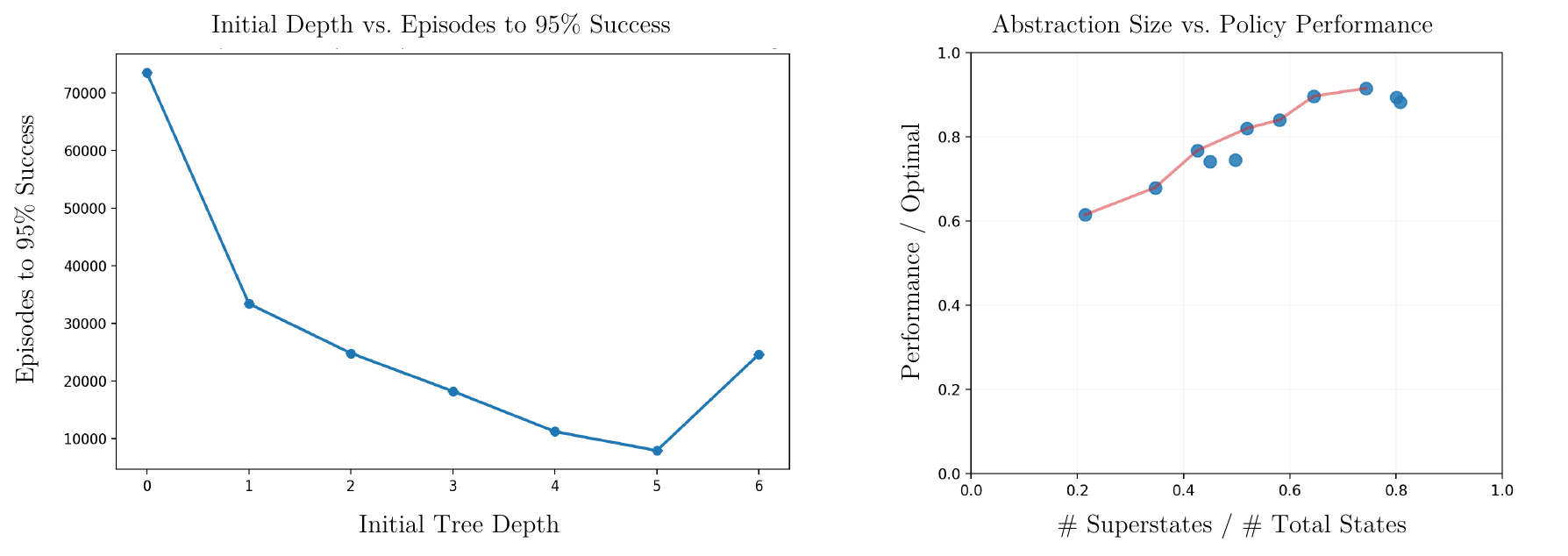}
    \vspace{-0.1in}
    \caption{\textit{Left:} Effect of the initial tree depth on the episodes required to reach a 95\% success rate.
    \textit{Right:} The Pareto frontier between abstraction size and policy performance.}
    \label{fig:pareto}
    \vspace{-0.1in}
\end{figure}

% \begin{figure}[t]
%     \centering
%     \includegraphics[width=0.9\linewidth]{figures/training_curves.png}
%     \vspace{-0.1in}
%     \caption{
%     Learning curves comparing tree-based actor–critic variants with the flat baseline.
%     }
%     \label{fig:training-curve-16}
%     \vspace{-0.2in}
% \end{figure}

\section{Conclusions}

We studied performance-driven abstractions for decision-making in large state spaces.
For a fixed partition, we established a performance guarantee that separates value-approximation error from the loss induced by the same-action-distribution (SAD) constraint.
Motivated by this analysis, we developed a multi-timescale learning framework that jointly optimizes the policy and a tree-based state abstraction.
The resulting algorithm adaptively refines and coarsens regions of the state space to balance performance and representational complexity.
Empirical results demonstrate that the learned hierarchical structure enables substantial state compression, accelerates replanning, and improves adaptability compared to flat actor–critic baselines.
A key future direction is to move beyond the current one-level greedy lookahead and establish theoretical guarantees for multi-level or globally optimal abstraction updates.
Another promising extension is to apply the framework beyond grid domains, including unstructured environments and joint state spaces in multi-agent systems.

% Several extensions remain promising.
% While the current formulation assumes a homogeneous action space and a predefined hierarchical tree (e.g., quadtree), the framework naturally suggests data-driven tree construction and topology-aware basis generation, such as proto-value functions.
% Moreover, although we approximate the performance certificate via a greedy one-step lookahead criterion, deriving explicit theoretical guarantees for this approximation would further strengthen the approach.
% Finally, extending the framework beyond grid domains—to unstructured environments or joint state spaces in multi-agent systems—offers a compelling direction for future research.

% References

\bibliography{references}

% \onecolumn

% \title{Performance-Driven Environment Abstraction with Multi-Timescale Learning\\(Supplementary Material)}
% \maketitle

\newpage
\appendix
\renewcommand{\thesection}{Appendix \Alph{section}}

\section{Proof of Theorem~\ref{thm:hier-sub-opt}}
\label{appdx-sec:thm-sub-opt-proof}
\thmaggopt*
\begin{proof}
    We first decompose the error $\norm{\Psi_\Gamma \bar V^*_{\Gamma,\mu^*} - V^*}$ into two parts. 
    \begin{align}
         \norm{\Psi_\Gamma \bar V^*_{\Gamma,\mu^*} -  V^*}_{\infty} &\leq \norm{\Psi_\Gamma \bar V^*_{\Gamma,\mu^*} -  \Psi_\Gamma\bar V_\opt}_\infty + \norm{\Psi_\Gamma \bar V_\opt - V^*}_\infty \nonumber\\
         &= \norm{\Psi_\Gamma \bar V^*_{\Gamma,\mu^*} -  \Psi_\Gamma \bar V_\opt}_\infty + \epsilon_\Gamma \nonumber\\
         &\leq \norm{\bar V^*_{\Gamma,\mu^*} -  \bar V_\opt}_\infty + \epsilon_\Gamma,
         \label{eqn:bound-1}
    \end{align}
    where we use the non-expansiveness of $\Psi_\Gamma$ for the last inequality.
    To bound the first term, we consider
    \begin{align*}
        \norm{\bar V_\opt - \bellop{\Gamma, \mu^*} \bar V_\opt}_\infty
        &\leq \norm{\Psi^{+}_{\Gamma,\mu^*} \circ\Psi_\Gamma  \bar V_\opt - \Psi^{+}_{\Gamma,\mu^*} \circ \bellop{\S} \circ \Psi_\Gamma  \bar V_\opt}_\infty + \underbrace{\norm{\Psi^{+}_{\Gamma,\mu^*} \circ \bellop{\S} \circ \Psi_\Gamma  \bar V_\opt - \bellop{\Gamma,\mu^*} \bar V_\opt}_\infty}_{\delta_\Gamma}\\
        &\stackrel{(\text{i})}{\leq} \norm{\Psi_\Gamma  \bar V_\opt -  \bellop{\S} \circ \Psi_\Gamma  \bar V_\opt}_\infty + \delta_\Gamma\\
        &\leq \norm{\Psi_\Gamma  \bar V_\opt - V^*}_\infty + \norm{V^* - \bellop{\S} \circ \Psi_\Gamma  \bar V_\opt}_\infty +\delta_\Gamma\\
        &\stackrel{(\text{ii})}{=}  \norm{\Psi_\Gamma  \bar V_\opt - V^*}_\infty + \norm{\bellop{\S} V^* - \bellop{\S} \circ \Psi_\Gamma  \bar V_\opt}_\infty +\delta_\Gamma \\
        &\stackrel{(\text{iii})}{\leq} \epsilon_\Gamma + \beta \epsilon_\Gamma + \delta_\Gamma,
    \end{align*}
    where inequality (i) follows from the non-expansiveness of $\Psi^+_{\Gamma,\mu^*}$; equality (ii) uses the fact that $V^*$ is the fixed point of the Bellman operator $\bellop{\S}$, i.e., $\bellop{\S} V^* = V^*$, and inequality (iii) is due to $\bellop{\S}$ being a $\beta$-contractive operator.
    
    Then, we have 
    \begin{align*}
        \norm{\bar V_{\Gamma,\mu^*}^* - \bar V_{\opt}}_\infty
        &\leq \norm{\bar V_{\Gamma,\mu^*}^* -  \bellop{\Gamma,\mu^*} \bar V_{\opt}}_\infty + \norm{\bellop{\Gamma, \mu^*} \bar V_{\opt} - \bar V_{\opt}}_\infty\\
        & = \norm{ \bellop{\Gamma, \mu^*} V_{\Gamma,\mu^*}^* -  \bellop{\Gamma, \mu^*} \bar V_{\opt}}_\infty + \norm{\bellop{\Gamma, \mu^*} \bar V_{\opt} - \bar V_{\opt}}_\infty \\
        &\leq \beta  \norm{\bar V_{\Gamma,\mu^*}^* - \bar V_{\opt}}_\infty + (1+\beta) \epsilon_\Gamma + \delta_\Gamma,
    \end{align*}
    which implies that 
    \begin{equation}
        \norm{\bar V_{\Gamma,\mu^*}^* - \bar V_{\opt}}_\infty \leq \frac{(1+\beta)\epsilon_\Gamma + \delta_\Gamma}{1-\beta}.
    \end{equation}
    Substituting into~\eqref{eqn:bound-1}, we obtain
    \begin{equation}
        \norm{\Psi_\Gamma \bar V^*_{\Gamma,\mu^*} -  V^*}_\infty \leq \frac{(1+\beta)\epsilon_\Gamma + \delta_\Gamma}{1-\beta} + \epsilon_\Gamma = \frac{2\epsilon_\Gamma + \delta_\Gamma}{1-\beta}.
    \end{equation}
\end{proof}

\section{Proof of Corollary~\ref{cor:hier-sub-opt}}
\label{appdx-sec:cor-sub-opt-proof}
\coraggopt*
\begin{proof}
    Notice that $\aggV^*_{\Gamma, \mu^*}$ is the fixed point of $\bellop{\Gamma,\mu^*}$ and $V^*$ is the fixed point of $\bellop{\S}$. 
    Thus, we have 
    \begin{align*}
        \norm{\Psi_\Gamma \aggV^*_{\Gamma, \mu^*} - V^*}_{\infty} &=
        \norm{\Psi_\Gamma \bellop{\Gamma,\mu^*}\aggV^*_{\Gamma, \mu^*} - \bellop{\S}V^*}_{\infty} \\
        &\leq \norm{\Psi_\Gamma \bellop{\Gamma,\mu^*}\aggV^*_{\Gamma, \mu^*} - \bellop{\S } \Psi_\Gamma \aggV^*_{\Gamma, \mu^*}}_{\infty} +
        \norm{\bellop{\S } \Psi_\Gamma \aggV^*_{\Gamma, \mu^*} -
        \bellop{\S}V^*}_{\infty} \\
        &\leq \norm{\Psi_\Gamma \bellop{\Gamma,\mu^*}\aggV^*_{\Gamma, \mu^*} - \bellop{\S } \Psi_\Gamma \aggV^*_{\Gamma, \mu^*}}_{\infty} +
        \beta \norm{\Psi_\Gamma \aggV^*_{\Gamma, \mu^*} -
        V^*}_{\infty}.
    \end{align*}
    Rearranging terms yields
    \begin{equation*}
        \norm{\Psi_\Gamma \aggV^*_{\Gamma, \mu^*} - V^*}_{\infty} \leq \frac{\norm{\Psi_\Gamma \bellop{\Gamma,\mu^*}\aggV^*_{\Gamma, \mu^*} - \bellop{\S } \Psi_\Gamma \aggV^*_{\Gamma, \mu^*}}_{\infty}}{1-\beta}.
    \end{equation*}
    
\end{proof}

\section{Pseudo-code for the Learning Algorithm }
\label{apdx-sec:code}

\begin{algorithm}[H]
\SetAlCapFnt{\normalsize}   % Caption font
\SetAlCapNameFnt{\normalsize} % "Algorithm X" font
\caption{Expansion}
\SetKwFunction{FExpandNode}{ExpandNode}
\SetKwProg{Fn}{Function}{:}{}

\Fn{\FExpandNode{$\tree, \gamma, \aggQ, \aggV, \aggQ^\est$}}{
  $\tree \gets \expand{\tree, \gamma}$\;
  $\aggQ(\gamma', a) \gets (1-\rho)\aggQ(\gamma, a) + \rho \aggQ^\est(\gamma', a)$ \qquad for all $\gamma' \in \children{\gamma}$ and $a \in \A$ \;
  $\aggV(\gamma') \gets (1-\rho)\aggV(\gamma) + \rho \max_{a\in \A} \aggQ^\est(\gamma', a)$ \quad for all $\gamma' \in \children{\gamma}$ \;
}

\end{algorithm}

\begin{algorithm}[H]
\SetAlCapFnt{\normalsize}   % Caption font
\SetAlCapNameFnt{\normalsize} % "Algorithm X" font
\caption{Collapse}
\SetKwFunction{FCollapseNode}{CollapseNode}
\SetKwProg{Fn}{Function}{:}{}

\Fn{\FCollapseNode{$\tree, \gamma, \aggQ, \aggV, \aggQ^\est$}}{
  $\tree \gets \collapse{\tree, \gamma}$\;
  $\aggQ(\gamma, a) \gets  (1-\rho)\sum_{\gamma' \in \children{\gamma}}\frac{1}{|\children{\gamma}|} \aggQ(\gamma', a) + \rho \aggQ^\est(\gamma, a)$ \qquad ~~ for all $a \in \A$ \;
  $\aggV(\gamma') \gets (1-\rho)\sum_{\gamma' \in \children{\gamma}}\frac{1}{|\children{\gamma}|}\aggV(\gamma') + \rho \max_{a\in \A} \aggQ^\est(\gamma, a)$ \;
}

\end{algorithm}

\begin{algorithm}[h!]
\SetAlCapFnt{\normalsize}   % Caption font
\SetAlCapNameFnt{\normalsize} % "Algorithm X" font
\caption{Tree-based Abstraction Learning}
\label{alg:abstraction-learning}
\SetKwFunction{FMain}{Main}
\SetKwFunction{FUpdateQest}{UpdateQest}
\SetKwFunction{FExpandCond}{CheckExpand}
\SetKwFunction{FCollapseCond}{CheckCollapse}

\SetKwProg{Fn}{Function}{:}{}
\Fn{\FMain{}}{
    Initialize tree $\tree$\;
    Initialize $\aggV_0(\gamma)$, $\aggQ_0(\gamma, a)$ for all $\gamma \in \leafNode{\tree}$ and $a \in \A$\;
    Initialize $\aggQ_0^\est(\gamma', a)$ for all $\gamma' \in \children{\gamma} \cup \parent{\gamma}$, $\gamma \in \leafNode{\tree}$\;
    Initialize learning rates $\xi_0, \zeta_0, \eta_0$\;
    \For{$\mathrm{ep}=0 \text{ to } \mathrm{ep}_{\mathrm{max}}$}{
    Initialize starting state $S_t$\;
    \While{$S_t$ is not terminal}{
        Find current superstate $\Gamma_t = \phi_{\tree}(S_t)$ and sample action $A_t$ according to $\psi_{\aggQ}(\Gamma_t, \cdot)$\;
        Take action $A_t$, observe reward $R_t$ and next state $S_{t+1}$ and superstate $\Gamma_{t+1}=\phi_{\tree} (S_{t+1})$\;
        Update Aggregated Actor-Critic according to~\eqref{eqn:agg-ac}\;
        % \begin{align*}
        %     \aggV_{t+1}(\Gamma_t) &=  (1-\xi_t) \aggV_{t}(\Gamma_t) + \xi_t (R_t + \beta \aggV_{t}(\Gamma_{t+1})),\\
        %     \aggQ_{t+1}(\Gamma_t, A_t) &= \proj{q} \Big((1-\zeta_t) \aggQ_{t}(\Gamma_t, A_t) + \zeta_t (R_t + \beta \aggV_{t}(\Gamma_{t+1}))\Big)
        % \end{align*}
        Perform \FUpdateQest{$\tree_\mathrm{input} = \tree,S_t, A_t, S_{t+1},\eta_t$}\;

        \If{$\Gamma_t$ is expandable}{
            Perform \FUpdateQest{$\tree_\mathrm{input}=\expand{\tree, \Gamma_t},S_t, A_t, S_{t+1},\eta_t$}\;
        }
        \If{$\parent{\Gamma_t}$ is collapsible}{
            Perform \FUpdateQest{$\tree_\mathrm{input}=\collapse{\tree, \parent{\Gamma_t}},S_t, A_t, S_{t+1},\eta_t$}\;
        }
        $t \gets t+1$\;
    }
    \If{$t \bmod \text{tree\_update\_frequency} = 0$}{
        % Initialize empty lists \texttt{expand\_node\_list}, \texttt{collapse\_node\_list} \;
        \For{$\gamma \in \leafNode{\tree}$}{
            \If{\FExpandCond{$\tree, \aggQ^\est, \gamma$}}{
                append $\gamma$ to \texttt{expand\_node\_list} \;
            }
            \If{\FCollapseCond{$\tree, \aggQ^\est, \parent{\gamma}$}}{
                append $\parent{\gamma}$ to \texttt{collapse\_node\_list} \;
            }
        
        }
        \FExpandNode{$\tree, \texttt{expand\_node\_list}, \aggQ, \aggV, \aggQ^\est$}\;
        \FCollapseNode{$\tree, \texttt{collapse\_node\_list}, \aggQ, \aggV, \aggQ^\est$}\;
}

% \textbf{Return:} Q-sets $\{\nsp{k}{i}\}_{k\leq \min\{\hmax, k_\infty\}, i \in \nodeset}$, initial allocation $\yinit = \vy^{(i_{-1}^*)}$, guaranteed breach time 
\textbf{Return:} Tree $\tree$, Policy $\Psi_{\aggQ}$
}
}

\Fn{\FUpdateQest{$\tree_\mathrm{input}, S_t, A_t, S_{t+1}, \eta_t$}}{
    $\Gamma_t = \phi_{\tree_\mathrm{input}}(S_t)$, 
    $\Gamma_{t+1} = \phi_{\tree_\mathrm{input}}(S_{t+1})$\;
    Update Q-estimate according to~\eqref{eqn:Q-est}\;
    % \[\aggQ_{t+1}^{\est}(\Gamma_{t}, A_t) = (1-\eta_t)\aggQ_{t}^{\est}(\Gamma_{t}, A_t) + \eta_t \Big(R_t + \beta \max_{a' \in \A} \aggQ_{t}^{\est}(\Gamma_{t+1}, a')\Big)\]
}
\Fn{\FExpandCond{$\tree_\mathrm{input}, S_t, A_t, S_{t+1}, \eta_t$}}{
    $\Gamma_t = \phi_{\tree_\mathrm{input}}(S_t)$ \;
    Evaluate $\hat \Delta \bar \epsilon_{\tree_\mathrm{input}, \Gamma_t}$ according to~\eqref{eqn:expansion-error-epsilon}\;
    \textbf{Return:} \texttt{True} if $\hat \Delta \bar \epsilon_{\tree_\mathrm{input}, \gamma} \geq 3 (1-\beta) \lambda$, \texttt{False} otherwise. 
    
}
\Fn{\FCollapseCond{$\tree_\mathrm{input}, S_t, A_t, S_{t+1}, \eta_t$}}{
    $\Gamma_t = \phi_{\tree_\mathrm{input}}(S_t)$, $\Gamma'_t = \parent{\Gamma_t}$, $\tree' = \collapse{\tree_\mathrm{input}, \Gamma'_t}$\;
    Evaluate $\hat \Delta \bar \epsilon_{\tree', \Gamma'_t}$ according to~\eqref{eqn:expansion-error-epsilon}\;
    \textbf{Return:} \texttt{True} if $\hat \Delta \bar \epsilon_{\tree', \Gamma'_t} < 3 (1-\beta) \lambda$, \texttt{False} otherwise. 
}
\end{algorithm}

\end{document}